\documentclass[a4paper]{cas-sc}

\usepackage[numbers]{natbib}

\usepackage{amsthm}
\usepackage[T1]{fontenc}
\usepackage{caption}
\usepackage{graphicx}
\usepackage{float} 
\usepackage{subcaption}
\usepackage{multirow}
\usepackage{algorithm}
\usepackage{algpseudocode}
\usepackage{bm}
\usepackage{amsfonts}
\usepackage{amsmath}
\usepackage{array} 
\newtheorem{lemma}{Lemma}

\def\tsc#1{\csdef{#1}{\textsc{\lowercase{#1}}\xspace}}
\tsc{WGM}
\tsc{QE}
\tsc{EP}
\tsc{PMS}
\tsc{BEC}
\tsc{DE}

\begin{document}
\let\WriteBookmarks\relax
\def\floatpagepagefraction{1}
\def\textpagefraction{.001}

\shorttitle{Variational Learning of Unobserved Confounders}

\shortauthors{Yonghe Zhao et~al.}

\title [mode = title]{VLUCI: Variational Learning of Unobserved Confounders for Counterfactual Inference}                      



%
\author[1]{Yonghe Zhao}[orcid=0000-0003-2613-7526]



\ead{yhzhao21@mails.jlu.edu.cn}



\affiliation[1]{organization={School of Artificial Intelligence, Jilin University},
    addressline={Qianjin Street 2699}, 
    city={Changchun},
    postcode={130012}, 
    state={Jilin},
    country={China}}

\author[1]{Qiang Huang}[style=chinese]



\author[1]{Siwei Wu}[style=chinese]



\author[2]{Yun Peng}[style=chinese]



\affiliation[2]{organization={Department of Data Analysis, Baidu},
    addressline={Shangdi 10th Street 10}, 
    city={Beijing},
    postcode={100085}, 
    country={China}}

\author[1]{Huiyan Sun}[style=chinese]








\begin{abstract}
Causal inference plays a vital role in diverse domains like epidemiology, healthcare, and economics. De-confounding and counterfactual prediction in observational data has emerged as a prominent concern in causal inference research. While existing models tackle observed confounders, the presence of unobserved confounders remains a significant challenge, distorting causal inference and impacting counterfactual outcome accuracy. To address this, we propose a novel variational learning model of unobserved confounders for counterfactual inference (VLUCI), which generates the posterior distribution of unobserved confounders. VLUCI relaxes the unconfoundedness assumption often overlooked by most causal inference methods. By disentangling observed and unobserved confounders, VLUCI constructs a doubly variational inference model to approximate the distribution of unobserved confounders, which are used for inferring more accurate counterfactual outcomes. Extensive experiments on synthetic and semi-synthetic datasets demonstrate VLUCI's superior performance in inferring unobserved confounders. It is compatible with state-of-the-art counterfactual inference models, significantly improving inference accuracy at both group and individual levels. Additionally, VLUCI provides confidence intervals for counterfactual outcomes, aiding decision-making in risk-sensitive domains. We further clarify the considerations when applying VLUCI to cases where unobserved confounders don't strictly conform to our model assumptions using the public IHDP dataset as an example, highlighting the practical advantages of VLUCI.
\end{abstract}




\maketitle

\section{Introduction}
In recent years, causal effect inference has attracted increasing attention across various domains, including but not limited to epidemiology, healthcare, and economics\cite{NIPS2017_6a508a60,2016Causal,johansson2018learning}. Compared to correlation, causality represents a more fundamental relationship between variables, revealing the directionality and determinacy\cite{imbens2015causal}. Randomized controlled trials (RCTs) are widely regarded as an effective means of exploring causality\cite{JudeaCausality}, however the complete randomness of RCTs limits their applicability in certain scenarios\cite{2016Assessing,2011Unexpected}. How to conduct causal inference directly from collected observational data is a research topic of widespread concern.

The counterfactual inference has become an ideal standard for causal inference from observational data\cite{morgan2015counterfactuals}, especially for individual treatment effect (ITE) estimation. One of the key challenges in counterfactual inference is the influence of the confounders which represent a class of variables that affect both the treatment and outcome variables\cite{imbens2015causal}. Taking an binary medical treatment as an example, as shown in Fig.~\ref{con}(left), treatment $A$ and $B$ represent two different treatment options for a certain disease; The outcome is whether the patient survives or not; The age of the patient is a typical confounder. In practice, physicians usually select a treatment option based on the age stage and physical condition of the patient, which may result in differences in the distribution of patient's age between the treatment group of $A$ and $B$, as shown in Fig.~\ref{con}(right). This difference leads to unsatisfactory accuracy of inference about patient survival when using an inferential model $S_{A}$ fitted by treatment $A$ to infer the survival of the treatment $B$, which is similar to the domain adaptation problem\cite{2018LearningWeighted}.
\begin{figure}
\centering
\includegraphics[width=0.95\textwidth]{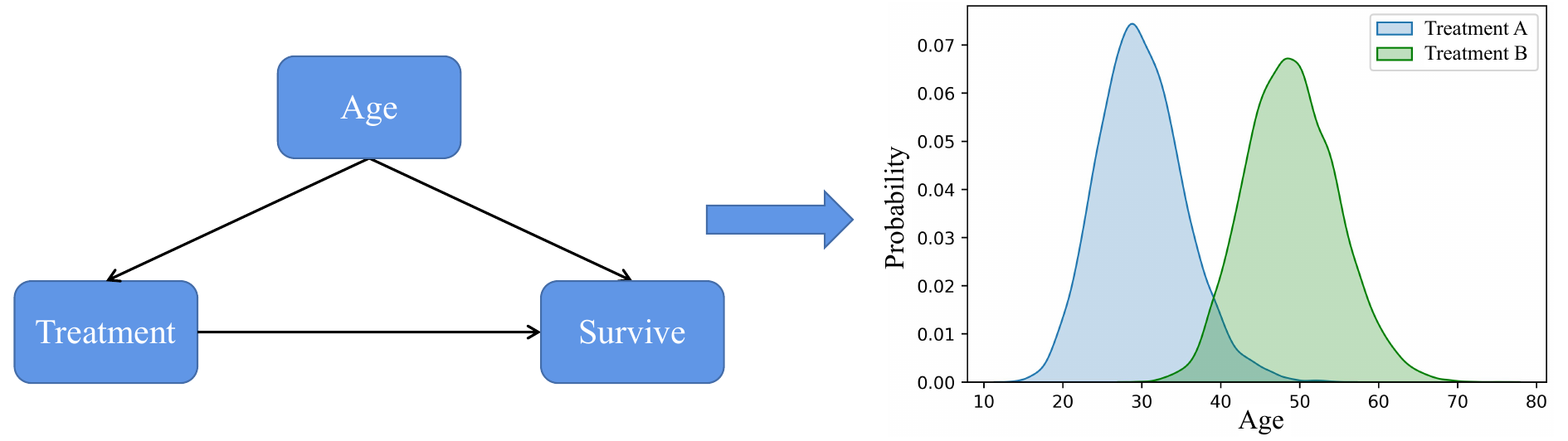}
\caption{Confounding Bias: the distributions of the confounders differed between treatment groups.}
\label{con}
\end{figure}
Various models for de-confounding of observed confounders are proposed under the unconfoundedness assumption, which refers to the absence of unobserved confounders\cite{imbens2015causal}, including reweighting\cite{rosenbaum1983central,2019Robust,lee2011weight,Austin2011An,imai2014covariate}, matching\cite{L2021Combining,2017Informative,JMLR21}, causal trees\cite{2010BART,Hill2011Bayesian,2017Estimation}, confounding balanced representation learning\cite{2018LearningWeighted,johansson2018learning,schwab2019perfect,2019AdversarialDu}, etc. Nevertheless, validating the unconfoundedness assumption can be challenging from observational data. This poses significant difficulties in implementing counterfactual inference.

There are two primary strategies to get rid of the dilemma when the unconfoundedness assumption is invalid. The first involves seeking instrumental variables\cite{baiocchi2014instrumental,2020Valid}. However, the detection of suitable instrumental variables satisfying stringent independent assumptions is also a challenging task\cite{hartford2017deep,hemani2018evaluating}. The second approach entails directly quantifying unobserved confounders. Generative models, particularly Variational Autoencoder (VAE) models\cite{kingma2019introduction}, possess the ability to learn the genuine data distribution and provide probabilistic descriptions of data in latent spaces. Thus, utilizing VAE models to estimate the distribution of unobserved confounders is a viable option. However, current studies focusing on the quantification of unobserved confounders are relatively limited. The CEVAE is a well-known model that quantifies latent confounders using proxies\cite{NIPS2017Latent}, making it crucial to identify proxy variables with strong representational power. Nevertheless, it is worth noting that not all unobserved confounders can be represented by proxy variables, and that cannot be learned from covariates also need to be considered. Fig. \ref{difference_con} illustrates that CEVAE-like methods concentrate on the latent confounders denoted by $L$. However, we are concerned with the unobserved confounders denoted by $C^{u}$ in Fig. \ref{difference_con}, which are independent of the observed covariates $X$. 

In this paper, the VLUCI model is proposed to estimate the distribution of unobserved confounders $C^{u}$. Unlike CEVAE-like methods, VLUCI is designed to relax the untestable assumption of unconfoundedness by providing unobserved confounding information beyond covariates, thereby improving the accuracy of counterfactual inferences. Specifically, we consider the generative mechanisms of both treatment and outcome variables to construct an interacting doubly variational generative model. The primary contributions of this paper are as follows:
\begin{figure}[ht]
\centering
\includegraphics[width=0.5\textwidth]{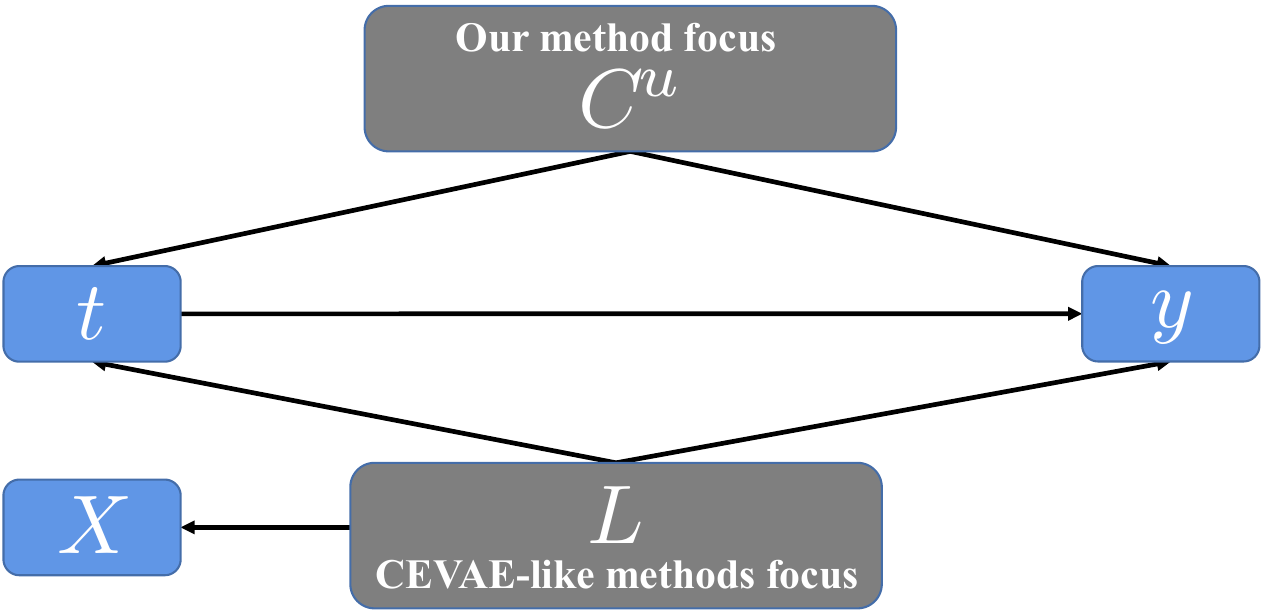}
\caption{The difference between the latent confounders $L$ emphasized by CEVAE-like methods and the unobserved confounders $C^{u}$ focused on in this paper. In the CEVAE-like methods, $X$ denote the proxy variables for $L$, whereas in this paper $X$ represent the observed confounding covariates.}
\label{difference_con}
\end{figure}
\begin{itemize}
\item By combining causal identifiability and partial regression theory, we present a theoretical proof of the unbiased estimation of the directly causal influence of covariates in both the treatment and outcome variables. Moreover, we propose a method for obtaining the treatment and outcome variables with the effects of covariates decomposed.
\item We develop a doubly variational inference model, VLUCI, to estimate the distribution of unobserved confounders that are independent of covariates, thus relaxing the untestable assumption of unconfoundedness.
\item We perform extensive experiments on synthetic datasets by combining VLUCI with existing state-of-the-art counterfactual inference methods. The results demonstrate the superior performance of VLUCI in inferring the distributions of unobserved confounders which significantly improve the counterfactual inference accuracy of existing methods.
\item Illustrated by the utilization of the public dataset IHDP, we engage in a discussion concerning the considerations and exceptional merits of applying VLUCI in scenarios where the actual tasks partially diverge from the assumptions articulated within this manuscript.
\end{itemize}
\section{Related Works}
Numerous methods for counterfactual inference have been proposed within the potential outcomes framework (POF), such as reweighting\cite{rosenbaum1983central,2019Robust,lee2011weight,Austin2011An,imai2014covariate}, stratification\cite{2015Principal,hullsiek2002propensity}, matching\cite{L2021Combining,2017Informative,JMLR21}, decision tree-based approaches (BART, C Forest)\cite{2010BART,2017Estimation}, and meta-learning\cite{K2017Meta,1988Root}, among others. These methods involve partitioning observational data with the common goal of approximating RCTs for counterfactual inference and then aggregating the results across subdomains using a weighted average.

Representation learning approaches have recently demonstrated their superiority for counterfactual inference\cite{zeyu2023causal}. These approaches differ from the methods of partitioning original data by learning deep representations of the original data\cite{B0Analysis}. By leveraging the powerful representation capability of neural networks, the representation learning-based methods obtain relatively balanced covariate representations\cite{2018LearningWeighted}. Notably, BNN\cite{johansson2018learning}, TARNET and CFR\cite{shalit2017estimating} models have been proposed for learning balanced covariate representations by minimizing the distance between the covariates of different treatment groups in the representation space. Furthermore, Schwab et al. developed the PM algorithm, a representation learning method based on data augmentation, which improves the equilibrium of training samples by finding the best matching sample within the training batch based on the propensity score\cite{schwab2019perfect}. Additionally, there are several counterfactual inference networks based on propensity score adjustment and multi-task representation learning, such as Propensity Dropout\cite{alaa2017deep} which regularizes each sample's weight based on treatment propensity, and Causal Multi-task Gaussian Processes for ITE estimation\cite{NIPS2017_6a508a60}. While the representation learning methods eliminate observed confounders and selective bias to some extent, the models depend on the untestable assumptions of unconfoundedness and sufficient overlap of covariates in observational data.

However there is limited literature on relaxing the unconfoundedness assumption, generative learning is considered to be a feasible approach. The CEVAE is a generative method for counterfactual inference that assumes the presence of proxy variables for unobserved confounders\cite{NIPS2017Latent}. These proxy variables are used to generate the joint distribution of treatment variables, unobserved confounders, and counterfactual outcomes. Yoon J. et al. introduced GANITE\cite{yoon2018ganite}, another generative method for estimating ITE. The method divides the ITE prediction model into two parts, namely counterfactual prediction and ITE estimation, both of which are learned using a GAN model. In a sense, the reconstruction error of the factual outcome transforms the GAN model into a standard auto-encoder. However, without this guidance, the GAN model would struggle to converge. Unfortunately, to the best of our knowledge, apart from CEVAE-like methods, other models require the assumption of unconfoundedness. Furthermore, the rationality of proxy variables for unobserved confounders in CEVAE-like models also need to be rigorously verified. 
\section{Method}
\subsection{Problem Setting}
\subsubsection{Symbol Description}
Referring to the  basic of causal inference in POF\cite{imbens2015causal}, the dataset, denoted as $\mathcal{D} = \{X_{i},t_{i},y_{i}\}_{i = 1}^{N}$, encompasses $N$ observational samples that are independent and identically distributed. Where, $X=\{X^{1},X^{2},\dots,X^{P}\} \in \mathbb{R}^{P}$ represents the $P$ covariates; $t \in \mathcal{T}$, $\mathcal{T}$ indicates treatment space. Each value of the treatment variable, $t_{i}$, corresponds to a potential outcome denoted by $y(t_{i})\in \mathbb{R}$. The potential outcomes can be specifically divided into the factual outcome $y^{f}$ and the counterfactual outcome $y^{cf}$. In addition to the observed $\mathcal{D}$, $C^{u}$ represents unobserved confounders in this paper.

\subsubsection{Assumptions} 
Different from the related literature which requires the satisfaction of the unconfoundedness assumption, the VLUCI relaxes this strong assumption. The SUTV and positivity assumptions that need to be satisfied in the proposed model are described below\cite{imbens2015causal}. 

The SUTV assumption includes: firstly, the potential outcome of each individual is not affected by the treatment of any other individual, in other words, individuals are independent; secondly, there is no measurement error in the factual observational outcome. 

The positivity assumption, commonly referred to as the overlap assumption, posits that each covariate can be assigned to any treatment with a non-zero probability, specifically $p(t|X=x) >0, \forall\ t \in \mathcal{T}, x \in X$. The purpose of counterfactual inference is to assess differences across treatments, and the model is meaningless if some treatments can not be observed or are not meaningful. 

In addition, based on the sample space described above, it is assumed that a corresponding causal structure, as depicted in Fig.~\ref{CausalStructure}, can provide insight into the generative mechanism behind the observational data. Naturally, due to the limitation of the observational data, the causal relationship between the treatment and the potential outcome is still influenced by unobserved confounders $C^{u}$, except for the observed covariates $X$. This causal structure emphasizes that the $C^{u}$ cannot be explained by the $X$.
\begin{figure}
\centering
\includegraphics[width=0.8\textwidth]{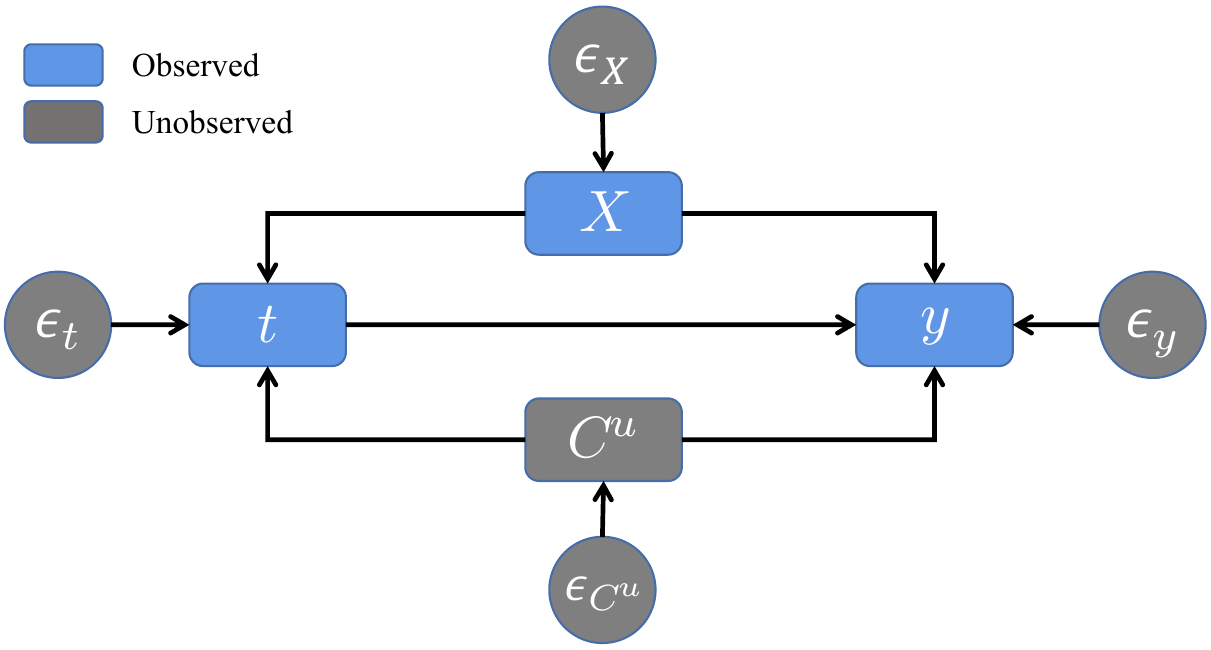}
\caption{The causal structure graph. Where $t$, $X$, and $y$ in blue squares represent observed treatment, covariables, and outcome. While $C^{u}$ and $\epsilon$ in gray squares represent unobserved confounders and random noise. } 
\label{CausalStructure}
\end{figure}

The causal structure shown in Fig.~\ref{CausalStructure} implies a latent assumption that the observed covariates $X$ and unobserved confounders $C^{u}$ are pre-treatment variables\cite{teixeira2005review,yao2020survey}, i.e., there are no causal paths from treatment $t$ to $X$ and $C^{u}$. 
\subsubsection{Key Objective} 
The structural equations correspond to Fig.~\ref{CausalStructure} are shown in Eq.~\eqref{struc_equa}. In this structure, $X$ and $C^{u}$ are exogenous variables that are entirely determined by random noise. The treatment $t$ is influenced by a combination of $X$, $C^{u}$, and $\epsilon_{t}$. Similarly, the potential outcome $y$ is influenced by a combination of $X$, $C^{u}$, $t$, and $\epsilon_{y}$. 
\begin{equation}
\label{struc_equa}
\centering
\begin{aligned}
&X=f_{X}\left(\epsilon_{X}\right) \\
&C^{u}=f_{C^{u}}\left(\epsilon_{C^{u}}\right)  \\
&t=f_{t}(X,C^{u},\epsilon_{t};\theta_{X}^{t},\theta_{C^{u}}^{t})  \\
&y=f_{y}(X,C^{u},t,\epsilon_{y};\theta_{X}^{y},\theta_{C^{u}}^{y},\theta_{t}^{y})  \\
\end{aligned}
\end{equation}

In the dataset $\mathcal{D}$, for the $i$-th sample, only the factual outcome $y_{i}^{f}(t_{i})$ corresponding to the $t_{i}$ is accessible, while the other counterfactual outcomes $y_{i}^{cf}(C_{\mathcal{T}}t_{i})$ are unobserved, where $C_{\mathcal{T}}t_{i}$ represent the complement of $t_{i}$ with respect to $\mathcal{T}$. The primary task of counterfactual inference is to infer $y_{i}^{cf}(C_{\mathcal{T}}t_{i})$. Furthermore, clarifying the structural equation $f_{y}$ for $y$ in Eq.~\eqref{struc_equa} is equivalent to obtaining counterfactual outcomes $y_{i}^{cf}$. Moreover, based on the structural equation $f_{y}$, the causal effects of any variable in $f_{y}$ on $y$ can be evaluated. Therefore, the key objective of this paper is to estimate the structural equation $f_{y}$ in Eq.~\eqref{struc_equa} based on observational data.

\subsection{A Practical Path for Estimation of the Structural Equation $f_{y}$}
For convenience of proof, we first decompose $f_{y}$ in Eq.~\eqref{struc_equa}. Since the $X$ and $C^{u}$ are pre-treatment variables, there is no interaction term between $t$, $X$ and $C^{u}$ in $f_{y}$, and $f_{y}$ can be expanded as the additive term shown in Eq.~\eqref{struc_equa_pro}.
\begin{equation}
\label{struc_equa_pro}
\begin{aligned}
y&=f_{y}(X,C^{u},t,\epsilon_{y};\theta_{X}^{y},\theta_{C^{u}}^{y},\theta_{t}^{y})\\
&= f_{y}^{X}(X\theta_{X}^{y}) + f_{y}^{C^{u}}(C^{u}\theta_{C^{u}}^{y}) + f_{y}^{t}(t\theta_{t}^{y}) + \epsilon_{y}\\
\end{aligned}
\end{equation}

Judea Pearl introduced a criterion for determining the identifiability of the causal effect of a singleton variable on all other variables in causal graphs: "The causal effect of singleton variable $A$ on $B$ is identifiable if there is no unobserved confounders between $A$ and $A$’s children in the subgraph composed of the ancestors of $B$."\cite{2002Identify}. Applying this criterion, the causal effect of $t$ on $y$, cannot be identified due to the presence of unobserved confounders $C^{u}$, which is equivalent to the inability to accurately estimate $f_{y}^{t}$ in Eq.~\eqref{struc_equa_pro}. Therefore, without the instrumental variables, inferring the distribution of $C^{u}$ based on observational data is a necessary prerequisite for estimating the causal effects of the treatment. Variational inference is an effective method for inferring the distribution of unobserved variables. One of the key ideas of this paper is to use variational inference to recover the distribution of $C^{u}$ by leveraging the variables $t$ and $y$ that have causal relationships with $C^{u}$. It is worth noting that, in addition to $C^{u}$, $t$ and $y$ are still influenced by the covariate $X$. Therefore, before performing variational inference, the parts of $t$ and $y$ that are directly influenced by $X$, denoted as $t_{X}$ and $y_{X}$, should be decomposed. It is evident based on the above criterion that the directly causal effects of $X$ on $t$ and $y$ are identifiable. We will derive how to estimate the direct causal effects unbiasedly below. 

To summarize, we construct a practical path for estimation of the structural equation $f_{y}$ that contains three steps, as shown in Fig.~\ref{step}: (i) identifying the parts of $t$ and $y$ that are directly influenced by $X$ ($t_{X}$ and $y_{X}$); (ii) inferring the distribution of $C^{u}$ based on $t$ and $y$ with the influenced parts of $X$ removed ($t_{C^{u}}$ and $y_{(C^{u},t)}$); and (iii) estimating Eq.~\eqref{struc_equa_pro} based on the distribution of $C^{u}$ specified.
\begin{figure}
\centering
\includegraphics[width=1.0\textwidth]{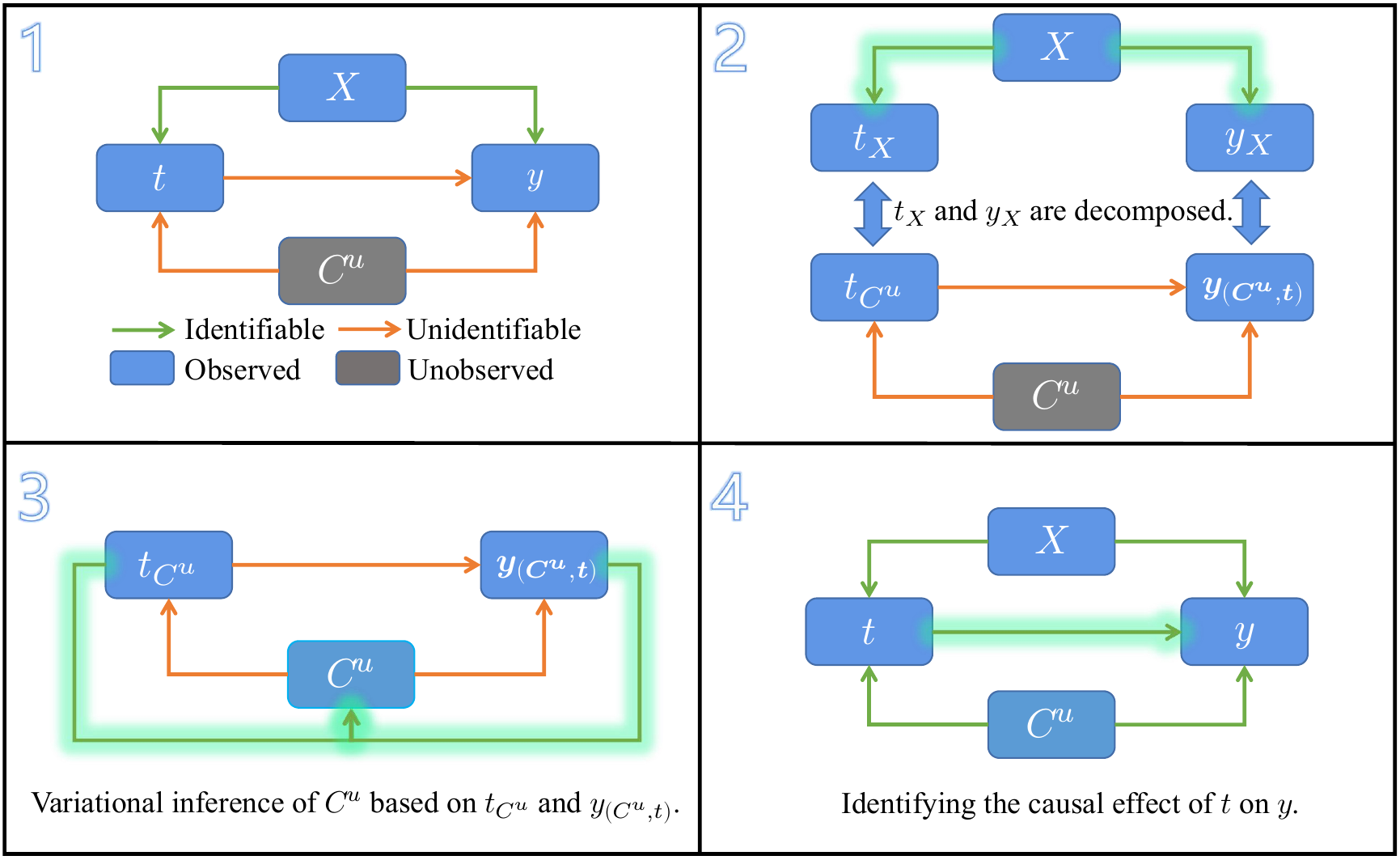}
\caption{The three steps of inference of the structural equation in Eq.~\eqref{struc_equa_pro}. } 
\label{step}
\end{figure}

\subsection{Unbiased Estimation of Directly Causal Influence $t_{X}$ and $y_{X}$ of the Observed Covariates $X$ on $t$ and $y$}
The $t_{X}$ can be directly estimated by fitting the observational data, as there are no other sources of deviation. To estimate $y_{X}$, we introduce Lemma 1, which is derived from the corresponding partial regression equation. 

\begin{lemma}
For the causal structure shown in Fig.~\ref{CausalStructure}, similar to the decomposition in Eq.~\eqref{struc_equa_pro}, let Eq.~\eqref{partial_func} represents the fittable partial regression equation for $y$ on the observed variables $X$ and $t$. In this case, the $X$-related term $f_{\hat{y}}^{X}(X\hat{\theta}_{X}^{y})$ in Eq.~\eqref{partial_func} is unbiased estimates of the term $f_{y}^{X}(X\theta_{X}^{y})$ in Eq.~\eqref{struc_equa_pro}.
\begin{equation}
\label{partial_func}
\begin{aligned}
\hat{y} &= f_{\hat{y}}(X,t;\hat{\theta}_{X}^{y},\hat{\theta}_{t}^{y})\\
&= \underbrace{f_{\hat{y}}^{X}(X\hat{\theta}_{X}^{y})}_{X-related} + \underbrace{f_{\hat{y}}^{t}(t\hat{\theta}_{t}^{y})}_{t-related}\\
\end{aligned}
\end{equation}
\end{lemma}
\begin{proof}
Let Eq.~\eqref{func_c_x} represents the functional relationship between unobserved $C^{u}$ and $X$ and $t$. 
\begin{equation}
\label{func_c_x}
\begin{aligned}
C^{u} &= f_{C^{u}}(X,t;\theta_{X}^{C^{u}},\theta_{t}^{C^{u}})\\
&= f_{C^{u}}^{X}(X\theta_{X}^{C^{u}}) + f_{C^{u}}^{t}(t\theta_{t}^{C^{u}})\ (\theta_{X}^{C^{u}} = \mathbf{0})\\
&= f_{C^{u}}^{t}(t\theta_{t}^{C^{u}})\\
\end{aligned}
\end{equation}
As $C^{u}$ and $X$ are independent, the parameters $\theta_{X}^{C^{u}}$ in Eq.~\eqref{func_c_x} are equal to $\mathbf{0}$. Next, substituting Eq.~\eqref{func_c_x} into Eq.~\eqref{struc_equa_pro} to eliminate the $C^{u}$ and obtaining Eq.~\eqref{y_eli_c}.
\begin{equation}
\label{y_eli_c}
\begin{aligned}
y&=f_{y}(X,f_{C^{u}}^{t}(t\theta_{t}^{C^{u}}),t,\epsilon_{y};\theta_{X}^{y},\theta_{C^{u}}^{y},\theta_{t}^{y})\\
&= \underbrace{f_{y}^{X}(X\theta_{X}^{y})}_{X-related} + \underbrace{f_{y}^{C^{u}}(f_{C^{u}}^{t}(t\theta_{t}^{C^{u}})\cdot\theta_{C^{u}}^{y}) + f_{y}^{t}(t\theta_{t}^{y})}_{t-related} + \epsilon_{y}\\
\end{aligned}
\end{equation}
Combining the $f_{y}$ in Eq.~\eqref{y_eli_c} and estimated $f_{\hat{y}}$ in Eq.\eqref{partial_func} by grouping like terms ($X$-related and $t$-related terms), according to the law of large numbers\cite{revesz2014laws}, leads to Eq.~\eqref{x_related}, thereby establishing the validity of Lemma 1.
\begin{equation}
\label{x_related}
\centering
\begin{aligned}
&\lim\limits_{N\to+\infty} \mathbb{E}(f_{\hat{y}}^{X}(X\hat{\theta}_{X}^{y})) = f_{y}^{X}(X\theta_{X}^{y}),\ \forall\ X \in \mathbb{R}^{P}\\
&\lim\limits_{N\to+\infty} \mathbb{E}(f_{\hat{y}}^{t}(t\hat{\theta}_{t}^{y})) = f_{y}^{C^{u}}(f_{C^{u}}^{t}(t\theta_{t}^{C^{u}})\cdot\theta_{C^{u}}^{y}) + f_{y}^{t}(t\theta_{t}^{y}),\ \forall\ t \in \mathcal{T}\\
\end{aligned}
\end{equation}
\end{proof}
A intuitive explanation of Lemma 1: Since $X$ and $C^{u}$ are independent in Figure~\ref{CausalStructure}, the absence of $C^{u}$ does not affect the estimation of direct causal effect of $X$ on $y$; Given that $C^{u}$ is confounders of $t$ and $y$, the causal effect of $C^{u}$ on $y$ will be reflected in the fitted $t$-related term $f_{\hat{y}}^{t}(t\hat{\theta}_{t}^{y})$ in Eq.~\eqref{partial_func}.

Based on Lemma 1 and the above discussion, we can unbiasedly identify and decompose the direct causal influence $t_{X}$ and $y_{X}$ from the observational data. Furthermore, we propose an interacting doubly variational generative model in the subsequent section to infer the distribution of unobserved confounders $C^{u}$.
\subsection{VLUCI}
\subsubsection{Overview Framework} 
The process of inferring unobserved confounders $C^{u}$ from the dataset $\{X,t,y\}$ can be viewed as a generative process. We construct the VLUCI model based on this fundamental concept, utilizing the latent causal structure illustrated in Fig.~\ref{CausalStructure}, as well as the unbiased segmentation of the causal effect for observed covariates and unobserved confounders. An overview of the model framework is presented in Fig.~\ref{Architecture}.
\begin{figure}[h]
\centering
\includegraphics[width=1.0\textwidth]{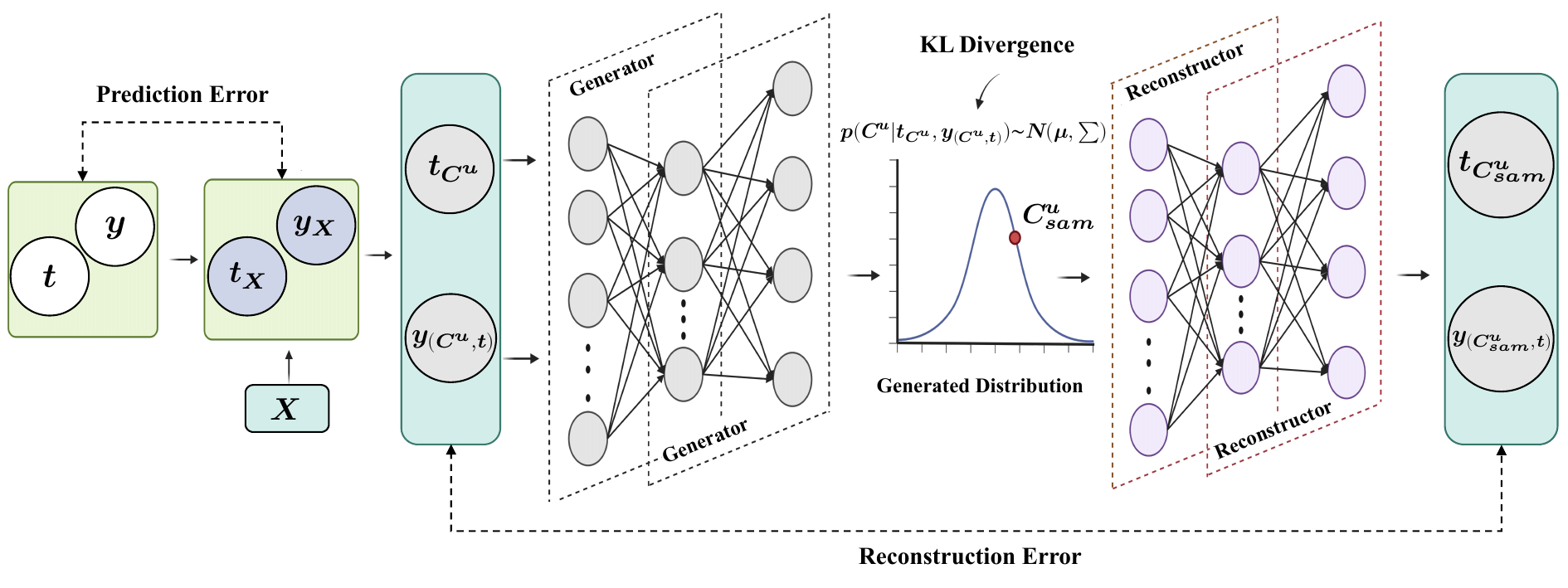}
\caption{The VLUCI model is comprised of three key components, namely, Prediction, Generator, and Reconstructor networks, which collectively form the overview framework of the model.} 
\label{Architecture}
\end{figure}

The model framework comprises three successive components. The first component is the prediction networks for estimating the direct causal influence $t_{X}$ and $y_{X}$ of $X$ on $t$ and $y$. This involves fitting the partial regression equation of $t$ and $y$, as shown in Eq.~\eqref{pred_net}.
\begin{equation}
\label{pred_net}
\hat{t} = f_{\hat{t}}(X;\theta_{f_{\hat{t}}});\ \hat{y} = f_{\hat{y}}(X,t;\theta_{f_{\hat{y}}})
\end{equation}
This component serves two key purposes. Firstly, the prediction networks function as a covariate feature filter by excluding unconfounding variables that are unrelated to $t$ and $y$ during the training process. Secondly, using Lemma 1, we are able to derive the $t$ and $y$ stripped of the causal influence of covariates $X$, as shown in Eq.~\eqref{stripped}. This effectively disentangles observed covariates from unobserved confounders.
\begin{equation}
\label{stripped}
\centering
\begin{aligned}
&\acute{t} = t_{C^{u}} = t - f_{\hat{t}}(X;\theta_{f_{\hat{t}}})\\
&\acute{y} = y_{(C^{u},t)} = y - f_{\hat{y}}(X,t = 0;\theta_{f_{\hat{y}}})\\
\end{aligned}
\end{equation}
Note that there is an operation of assigning $t$ to 0 in the derivation of $\acute{y}$ ($y_{(C^{u},t)}$) in Eq.~\eqref{stripped}. The reason for this operation is: For the estimated $\hat{y}$, the $X$-related term is unbiased and we only want to remove this part. In contrast, although the estimation of the $t$-related term is biased, we assign $t$ to 0 to retain this part of information.

The second component of the model framework is a variational generator network based on $\acute{t}$ and $\acute{y}$ for generating unobserved confounders which are assumed to subject to a multivariate Gaussian distribution $p(C^{u}|\acute{t},\acute{y}) \sim N(\mu,\Sigma)$\cite{doersch2016tutorial}. Upon completion of the second stage, the distribution of unobserved confounders is inferred. The final component is the reconstructor networks that utilize a reparameterization technique to reconstruct $\acute{t}$ and $\acute{y}$ based on the generated distribution $p(C^{u}|\acute{t},\acute{y})$.

The proposed VLUCI is designed to reflect the causal structure depicted in Fig.~\ref{CausalStructure}. Specifically, the model is cascaded in the forward propagation process, and interdependent in the backward propagation process. Within the model framework, a doubly reconstruction process of treatment and outcome variables interacts to infer the distribution of $C^{u}$.
\subsubsection{Optimization Objectives} 
The VLUCI model generates the fitted results $\hat{t}$ and $\hat{y}$ on $X$ in the prediction networks. Two corresponding prediction errors based on cross-entropy (CE) loss and mean square error (MSE), as shown in Eq.~\eqref{equ6_1}-\eqref{equ6_2}, are adopted. If the treatment variables are continuous, the loss $\mathcal{L}^{t}_{X}$ can be easily transformed into the MSE function.
\begin{equation}
\label{equ6_1}
\mathcal{L}^{t}_{X} = CE(t,\hat{t})= \frac{1}{N}\sum_{i=1}^{N}-[t_{i}log(\hat{t_{i}})+(1-t_{i})log(1-\hat{t_{i}})]
\end{equation}
\begin{equation}
\label{equ6_2}
\mathcal{L}^{y}_{X} = MSE(y,\hat{y})=\frac{1}{N}\sum_{i=1}^{N}(y_{i}-\hat{y}_{i})^{2}
\end{equation}

After fitting the prediction networks, we compute the subcomponents of $t$ and $y$ disentangled with $X$: $\acute{t}$ and $\acute{y}$, as shown in Eq.~\eqref{stripped}. The remainder of the model is a doubly variational learning with the target of inferring the distribution $p(C^{u}|\acute{t},\acute{y})$ of $C^{u}$. Since $C^{u}$ is unobserved, the distribution cannot be directly computed using the Bayesian posterior formula. Instead, we utilize the generator network to learn the posterior distribution $q(C^{u}|\acute{t},\acute{y})$ for approximating $p(C^{u}|\acute{t},\acute{y})$ by minimizing the Kullback-Leibler (KL) divergence between the two distributions. The KL divergence measures the distance between the two distributions and is defined in Eq.~\eqref{equ8}\cite{johnson2001symmetrizing}.
\begin{equation}
\label{equ8}
KL(q(C^{u}|\acute{t},\acute{y})||p(C^{u}|\acute{t},\acute{y})) = \sum q(C^{u}|\acute{t},\acute{y})log\frac{q(C^{u}|\acute{t},\acute{y})}{p(C^{u}|\acute{t},\acute{y})}
\end{equation}

However, there remains the unknown distribution $p(C^{u}|\acute{t},\acute{y})$ in Eq.~\eqref{equ8}, which is not optimizable. Referring to the methodology in \cite{2013arXiv13126114K}, the KL divergence in Eq.~\eqref{equ8} is approximated by decomposing it into three objective functions that can be learned through the network, as shown in Eq.~\eqref{equ9}-\eqref{equ11}. The detailed derivation of the decomposition is shown in Appendix A.
\begin{equation}
\label{equ9}
\mathcal{L}(q(C^{u}|\acute{t},\acute{y}),p(C^{u})) = KL(N(\mu,\Sigma)||N(0,I)))
\end{equation}
\begin{equation}
\label{equ10}
\mathcal{L}^{\acute{t}}_{C^{u}} = MSE(\acute{t},\acute{t}(C^{u}))
\end{equation}
\begin{equation}
\label{equ11}
\mathcal{L}^{\acute{y}}_{C^{u},t} = MSE(\acute{y},\acute{y}(C^{u},t))
\end{equation}

In summary, the five objective functions described in Eq.~\eqref{equ6_1}-\eqref{equ6_2} and Eq.~\eqref{equ9}-\eqref{equ11} are optimized to generate the distribution of the unobserved confounders. The specific algorithmic procedure for the VLUCI is outlined in Algorithm~\ref{alg:Framwork}.
\begin{algorithm}
\caption{The algorithm procedure of VLUCI.}\label{alg:Framwork}
\begin{algorithmic}
\Require \\
$\mathcal{D} = \{X^{i},t^{i},y^{i}\}_{i = 1}^{N}$; Training epochs: $M$; Batchsize: $B$;\\
Prediction networks: $f_{\hat{t}}(X;\theta_{f_{\hat{t}}})$ and $f_{\hat{y}}(X,t;\theta_{f_{\hat{y}}})$; \\
Generator network: $P(C^{u}|\acute{t},\acute{y}) \sim N(\mu(\acute{t},\acute{y};\theta_{\mu}),\Sigma(\acute{t},\acute{y};\theta_{\Sigma}))$; \\
Reconstructor networks: $f_{R-t}(C^{u};\theta_{f_{R-t}})$ and $f_{R-y}(C^{u},t;\theta_{f_{R-y}})$.
\Ensure 
\State Initialize Adam optimizer and all the parameters: $\theta_{f_{\hat{t}}}$, $\theta_{f_{\hat{y}}}$, $\theta_{\mu}$, $\theta_{\Sigma}$, $\theta_{f_{R-t}}$, $\theta_{f_{R-y}}$;
\For{$i \in [1,M]$}
    \State Randomly sample $B$ examples from $\mathcal{D}$;
    \State Update the parameters  $\theta_{f_{\hat{t}}}$ and $\theta_{f_{\hat{y}}}$ of the prediction networks by minimizing $\mathcal{L}^{t}_{X}$ and $\mathcal{L}^{y}_{X}$ shown as Eq.~\eqref{equ6_1}-\eqref{equ6_2};
    \State Update the parameters $\theta_{\mu}$ and $\theta_{\Sigma}$ of the generation network by minimizing Eq.~\eqref{equ9};
    \State Update the parameters $\theta_{f_{R-t}}$ and $\theta_{f_{R-y}}$ of the reconstruction networks by minimizing $\mathcal{L}^{\acute{t}}_{C^{u}}$ and $\mathcal{L}^{\acute{y}}_{C^{u},t}$ shown as Eq.~\eqref{equ10}-\eqref{equ11};
\EndFor
\State Result in well-trained $\theta_{f_{\hat{t}}}$, $\theta_{f_{\hat{y}}}$, $\theta_{\mu}$, $\theta_{\Sigma}$, $\theta_{f_{R-t}}$, $\theta_{f_{R-y}}$.
\end{algorithmic}
\end{algorithm}

\section{Experiments}
\subsection{Datasets}
In practice, researchers only have access to the factual outcome of a given treatment using real-world datasets, while other counterfactual outcomes are missing. In general, some synthetic or semi-synthetic datasets are applied in relevant literature\cite{Hill2011Bayesian,dehejia1999causal,NIPS2017Latent,johansson2018learning}. In this paper, we use two benchmark datasets to validate the performance of the proposed model with various state-of-the-art methods, including the synthetic dataset and the semi-synthetic dataset: Infant Health and Development Program (IHDP). 
\subsubsection{Synthetic Dataset} 
The pure causal effects of covariates and the distribution of unobserved confounders in a real-world dataset are inherently inaccessible. Therefore, a synthetic dataset is generated to validate Lemma 1 and to assess the performance of the proposed VLUCI in generating the distribution of unobserved confounders.

Based on the causal schema in Fig.~\ref{CausalStructure}, the synthetic dataset $\mathcal{D}_{synt}$ comprises of four components: $\mathcal{D}_{synt}= \{t,X,C^{u},y\}$. Firstly $C^{u}$ and $X$ are randomly generated as shown in Eq.~\eqref{equ17}-\eqref{equ18}, which are exogenous variables. For the convenience of illustration, $C^{u}$ is set as one-dimensional and $X$ is set as d-dimensional, where $d = 8$ in the experiment.
\begin{equation}
\label{equ17}
C^{u} \sim N(0,1)
\end{equation}
\begin{equation}
\label{equ18}
X \sim N(0^{1\times d},\frac{1}{2}(\Sigma + \Sigma^{T})), \Sigma \sim U((-1,1)^{d\times d})
\end{equation}

Subsequently, $t$ is randomly generated from a Bernoulli distribution with probabilities determined by $X$ and $C^{u}$, as shown in Eq.~\eqref{equ19}. Next, $y$ is generated by the Normal distribution shown in Eq.\eqref{equ20}. Where, $f(\cdot) = sigmoid(\cdot)$, $W_{\cdot \cdot} \sim U(1,4)$, $\epsilon_{t} \sim N(0,1)$ and $\epsilon_{y} \sim N(0,0.2)$. 
\begin{equation}
\label{equ19}
t\sim B(p(t=1));\ p(t=1) = 0.35\cdot f(C^{u}W_{C^{u}t}) + 0.6\cdot f(XW_{Xt}) + 0.05\cdot\epsilon_{t}
\end{equation}
\begin{equation}
\label{equ20}
y \sim N(\mu_{y},1);\ \mu_{y}= f(C^{u}W_{C^{u}y}) + f(tW_{ty}) + f(XW_{Xy}) +\epsilon_{y}
\end{equation}
\subsubsection{Semi-synthetic Dataset: IHDP} 
The IHDP is a RCT started in 1985 targeting low birth weight preterm infants, providing intensive high-quality child care and home visits for the treatment group\cite{brooks1992effects}. The covariates in the IHDP dataset are 6 continuous covariates and 19 binary covariates, including birth weight, head circumference, weeks born preterm, neonatal health index, sex, etc. The treatment variable is home visits with specialists, and the outcome variable is children's cognitive test scores. To obtain biased observational data, the treated assignment was then 'de-randomized' by removing children whose mother is non-white from the treatment group\cite{Hill2011Bayesian}. As a result, the treatment and control groups are no longer balanced and simple comparisons of outcomes lead to biased estimates of treatment effects.

\subsection{Metrics} 
To illustrate the improvement effect of VLUCI on various counterfactual inference models, the errors of individual treatment effect (ITE) and average treatment effect (ATE) are used to assess inferential accuracy at the individual and group levels, respectively, as shown in Eq.~\eqref{ite}-\eqref{ate}. 
\begin{equation}
\label{ite}
ITE_{i} = y^{i}(1)-y^{i}(0)
\end{equation}
\begin{equation}
\label{ate}
ATE = \mathbb{E}[y(1)-y(0)]
\end{equation}

The evaluation metric precision in the estimation of heterogeneous effects (PEHE) for the performance of ITE estimates\cite{Hill2011Bayesian}, $\epsilon_{PEHE}$, as shown in Eq.~\eqref{equ211}, is computed by determining the mean square error between the predicted and actual ITE.
\begin{equation}
\label{equ211}
\epsilon_{PEHE} = \frac{1}{N}\sum_{i =1}^{N}([y^{i}(1)-y^{i}(0)]-[\hat{y}^{i}(1)-\hat{y}^{i}(0)])^{2}
\end{equation}
Where $y^{i}$ represents the ground-truth outcome, and $y^{i}(1)-y^{i}(0)$ reflects the actual value of the ITE. On the other hand, $\hat{y}^{i}$ denotes the predicted outcome, and $\hat{y}^{i}(1)-\hat{y}^{i}(0)$ represents the predicted value of the ITE.

In addition, $\epsilon_{ATE}$, absolute error of ATE is used to measure model performance at the group level, as shown in Eq.~\eqref{equ222}.
\begin{equation}
\label{equ222}
\epsilon_{ATE} = |ATE - \hat{ATE}|
\end{equation}
\subsection{Experimental Design}
In this section, we conduct experiments to demonstrate the effectiveness of the proposed model. Specifically, we aim to answer the following research questions:
\begin{itemize}
\item RQ1: Can the proposed model estimate the direct causal effects of $X$ on $t$ and $y$ unbiasedly? That is, is Lemma 1 experimentally plausible?
\item RQ2: Are the deduced unobserved confounders $\hat{C}^{u}$ consistent with the ground-truth $C^{u}$?
\item RQ3: To what extent can the inferred distribution of $C^{u}$ improve the accuracy of various state-of-the-art counterfactual inference models?
\item RQ4: What inferences about $C^{u}$ will the model output when the situation does not conform to the causal generation mechanism shown in Fig.~\ref{CausalStructure}?
\end{itemize}

For RQ1 and RQ2, the pure causal effects of $X$ on $t$ and $y$, $t_X = 0.6\cdot f(C^{u}W_{C^{u}t})$ and $y_X = f(C^{u}W_{C^{u}y})$, and $C^{u}$ as ground-truth is not used for training. The unbiased estimation and inferential performance of the proposed model is verified by comparing the predicted $\hat{t}_X$, $\hat{y}_X$ and $\hat{C}^{u}$ of the model with the corresponding ground-truth values. 

For RQ3, we distinguish whether to add learned $\hat{C}^{u}$ to report the inferred performance of various state-of-the-art counterfactual inference models on synthetic dataset, including inverse probability weighting (IPW)\cite{rosenbaum1983central}, causal forests (C Forest)\cite{2017Estimation}, treatment-agnostic representation network (TARNET)\cite{shalit2017estimating}, counterfactual regression (CFR)\cite{shalit2017estimating}, causal effect variational autoencoder (CEVAE)\cite{NIPS2017Latent}, and generative adversarial nets for inference of ITE (GANITE)\cite{yoon2018ganite}. The synthetic dataset is randomly divided into training/test sets according to the percentages of 80/20. Under this division criterion, the models to be evaluated are repeated 100 times to record the mean and standard error of the evaluation metrics on both the training and test set. To ensure a fair comparison of the comparison results, the hyperparameters of each model remain consistent regardless of the addition of the learned $\hat{C}^{u}$ to participate in the training.

For RQ4, the proposed model is applied to the semi-synthetic dataset IHDP to analyse inferred results in cases of inconsistency with hypothetical scenarios.
\subsection{Results}
\subsubsection{Unbiased estimations of the direct causal effects of $X$ on $t$ and $y$ (RQ1).}
Refer to the model framework shown in Fig.~\ref{Architecture}, we first fit prediction networks of covariates $X$ on $t$ and $y$ in Eq.~\eqref{pred_net}, respectively. Then, $\hat{t}_X = \hat{t}(X)$ and $\hat{y}_X = \hat{y}(X, t = 0)$ are treated as estimated causal effects for comparison with the actual ones. Fig.~\ref{fig:x_t_yCase1} and Fig.~\ref{fig:x_t_yCase2} show the comparison of the predicted and ground-truth causal effects of $X$ on $t$ and $y$, respectively. It is obvious that approximately unbiased causal effect estimations of $X$ are derived, namely, the correctness of Lemma 1 is experimentally demonstrated.
\begin{figure}[h]
	\centering
    \begin{subfigure}[t]{0.49\textwidth}
           \centering
           \includegraphics[width=\textwidth]{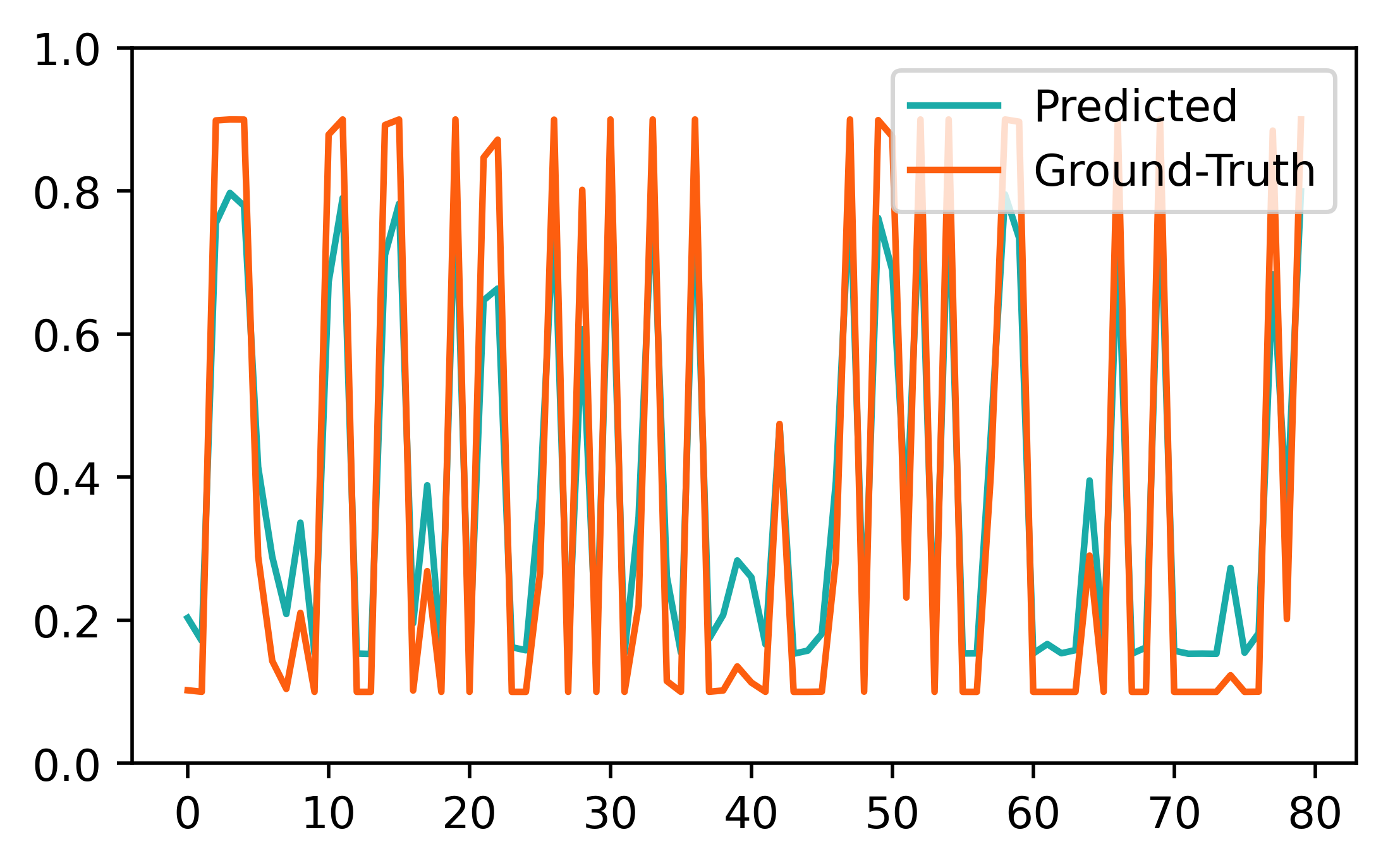}
            \caption{The direct causal effects of $X$ on $t$.}
            \label{fig:x_t_yCase1}
    \end{subfigure}
    \begin{subfigure}[t]{0.49\textwidth}
            \centering
            \includegraphics[width=\textwidth]{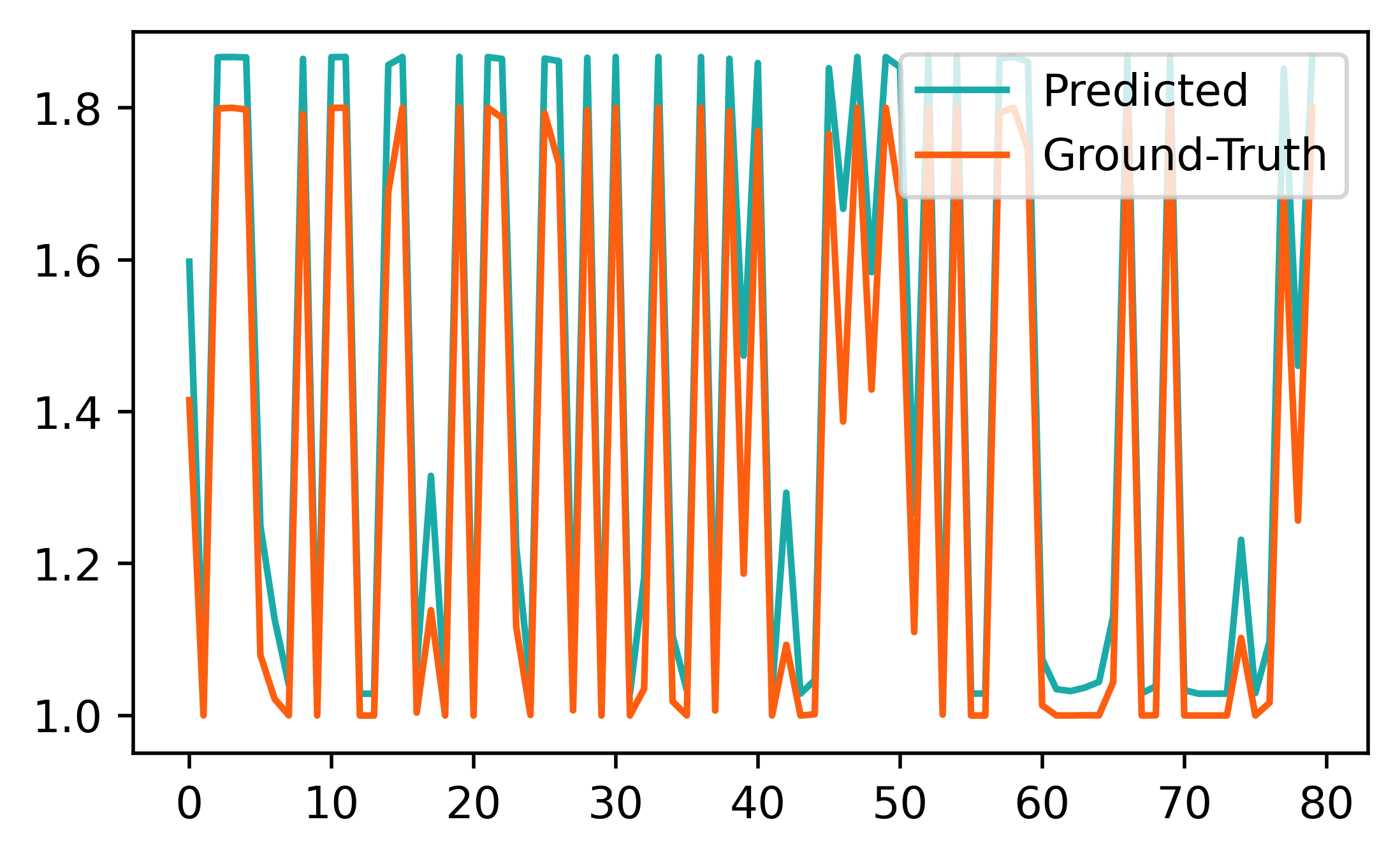}
            \caption{The direct causal effects of $X$ on $y$.}
            \label{fig:x_t_yCase2}
    \end{subfigure}
    \caption{The predicted and ground-truth direct causal effects of $X$ on $t$ and $y$.}
    \label{FIG:x_t_y}
\end{figure}
\subsubsection{Inference of unobserved confounders $C^{u}$ (RQ2).} 
Based on the estimated effects $\hat{t}_X$ and $\hat{y}_X$ of $X$, we obtain $t_{C^{u}}$ and $y_{(C^{u},t)}$ driven by $C^{u}$ shown in Eq.~\eqref{stripped}, which are used to infer the distribution of $C^{u}$. The estimation of the distribution of $C^{u}$ is provided in Fig.~\ref{UnconfundingInference}. The results illustrate that the distribution of $C^{u}$ predicted by the proposed model coincides with the unobserved confounders which are not involved in the training process. 
\begin{figure}[h]
\centering
\includegraphics[width=0.8\textwidth]{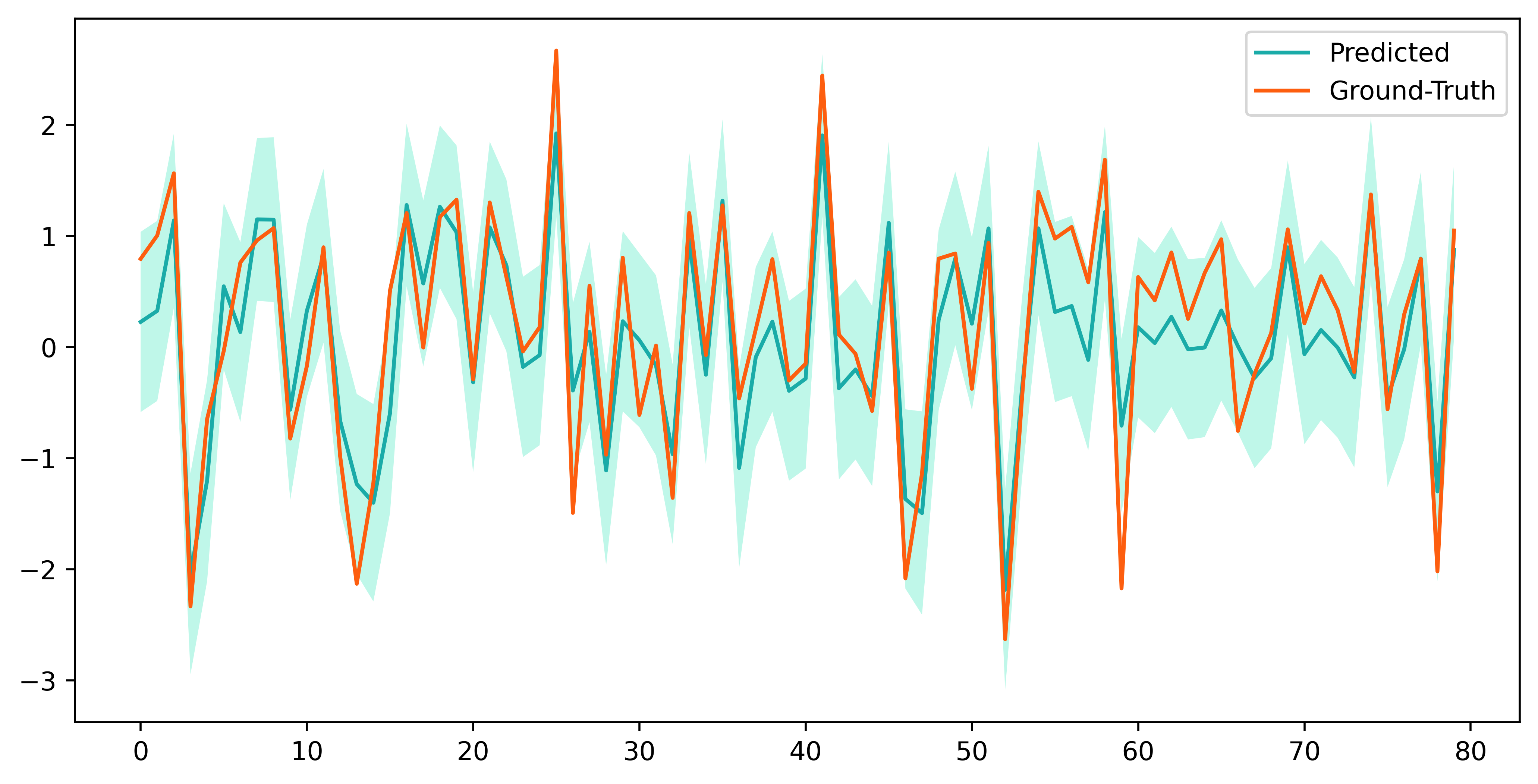}
\caption{The predicted distribution for the unobserved confounders of VLUCI, and the corresponding ground-truth values in $\mathcal{D}_{synt}$. Where the shading represents the $3\sigma$ confidence interval.} 
\label{UnconfundingInference}
\end{figure}
\subsubsection{The effectiveness of learned $\hat{C}^{u}$ on the improvement of various models (RQ3).}
To demonstrate the improvement of counterfactual inference via $\hat{C}^{u}$ learned by VLUCI, we independently train various prevalent models with identical hyperparameters on both the raw data $\{X_{i},t_{i},y_{i}\}_{i = 1}^{N}$ and the augmented data $\{X_{i},\hat{C}^{u}_{i},t_{i},y_{i}\}_{i = 1}^{N}$. The inferential results of these models are subsequently reported, as depicted in Table~\ref{ATE-PEHE}. It is clear that adding $\hat{C}^{u}$ to the training process significantly improve the prediction accuracy of various counterfactual inference methods at both individual and group levels.
\begin{table}[ht]
\centering
\caption{Performance of $\sqrt{\epsilon_{PEHE}}$ and $\epsilon_{ATE}$ estimation on synthetic datasets. Where $+\hat{C}^{u}$ represents the addition of the learned $\hat{C}^{u}$ to train counterfactual inference model.}
\resizebox{\linewidth}{!}{
\begin{tabular}{|l|cc|cc|cc|cc|}
\hline
\multirow{3}{*}{{\bfseries Methods}} &  \multicolumn{8}{|c|}{{\bfseries Metrics}}\\
\cline{2-9}
& \multicolumn{4}{|c|}{\bm{$\sqrt{\epsilon_{PEHE}}$}} & \multicolumn{4}{|c|}{\bm{$\epsilon_{ATE}$}} \\
\cline{2-9}
& Training  & $\bm{+\hat{C}^{u}}$ &  Test  & $\bm{+\hat{C}^{u}}$ & Training  & $\bm{+\hat{C}^{u}}$ &  Test  & $\bm{+\hat{C}^{u}}$\\
\hline
IPW & $.62\pm.01$ & \bm{$.55\pm.00$} & $.62\pm.01$ & \bm{$.55\pm.01$} & $.25\pm.01$ & \bm{$.12\pm.01$} & $.25\pm.02$ & \bm{$.12\pm.02$} \\
C Forest & $.73\pm.01$ & \bm{$.68\pm.01$} & $.73\pm.02$ & \bm{$.68\pm.01$} & $.46\pm.01$ & \bm{$.39\pm.01$} & $.46\pm.02$ & \bm{$.39\pm.02$} \\
CEVAE & $.76\pm.04$ & \bm{$.71\pm.03$} & $.78\pm.04$ & \bm{$.72\pm.03$} & $.51\pm.06$ & \bm{$.45\pm.04$} & $.53\pm.06$ & \bm{$.48\pm.04$} \\
GANITE & $.70\pm.02$ & \bm{$.62\pm.02$} & $.70\pm.02$ & \bm{$.62\pm.02$} & $.40\pm.03$ & \bm{$.34\pm.02$} & $.41\pm.04$ & \bm{$.36\pm.02$} \\
TARNET & $1.29\pm.00$ & \bm{$.90\pm.00$} & $1.30\pm.01$ & \bm{$.91\pm.00$} & $1.16\pm.01$ & \bm{$.75\pm.01$} & $1.17\pm.01$ & \bm{$.76\pm.01$} \\
CFR & $1.26\pm.00$ & \bm{$.88\pm.00$} & $1.27\pm.00$ & \bm{$.90\pm.00$} & $1.14\pm.01$ & \bm{$.73\pm.01$} & $1.15\pm.00$ & \bm{$.74\pm.00$} \\
\hline
\end{tabular}}
\label{ATE-PEHE}
\end{table}
\subsubsection{Analysis of possible scenarios regarding the learned $\hat{C}^{u}$ (RQ4).}
The proposed VLUCI is applied to the IHDP dataset to explore the possible potential meanings of the inferred $C^{u}$. We apply the NPCI package and set the parameter "$setting = B$" to generate the IHDP dataset. According to \cite{Hill2011Bayesian}, the outcome variable in this scenario is generated by Eq. \eqref{equ102}. That is, $y$ is completely determined by $X$ and $t$ and no unobserved confounder. The phenomenon is also illustrated by the unbiased fitting $\hat{y}(X,t)$ of $X$ and $t$ to $y$ in Fig.~\ref{fig:IHDPCase1}. 
\begin{equation}
\label{equ102}
y(t = 0) \sim N(exp(X+W)\beta_{B},1));\ y(t = 1) \sim N(X\beta_{B} + \omega^{s}_{B},1))
\end{equation}

However, the proposed model still inferred ${C}^{u}$ with practical implications. As shown in Fig.~\ref{fig:IHDPCase2}, the predicted ${C}^{u}$ is highly correlated but not equivalent to the treatment variable $t$. To obtain a biased dataset, during the generation of the IHDP dataset, the samples of non-white mother were removed from the treatment group and the variable indicating the mother's skin colour was hidden. As a result, the proportion of white mother in the treat and control groups was $100\%$ and $37\%$ respectively. It is worth emphasising that the predicted ${C}^{u}$ is in approximate coincidence with this actual ratio. In other words, the VLUCI reproduces the hidden variable of mother's skin colour precisely. The results on the IHDP dataset illustrate that the proposed VLUCI is capable of inferring not only unobserved confounders, but also other potential variables associated with $t$ or $y$, which contribute to a better understanding of the causal mechanisms underlying the data.
\begin{figure}[h]
	\centering
    \begin{subfigure}[t]{0.45\textwidth}
           \centering
           \includegraphics[width=\textwidth]{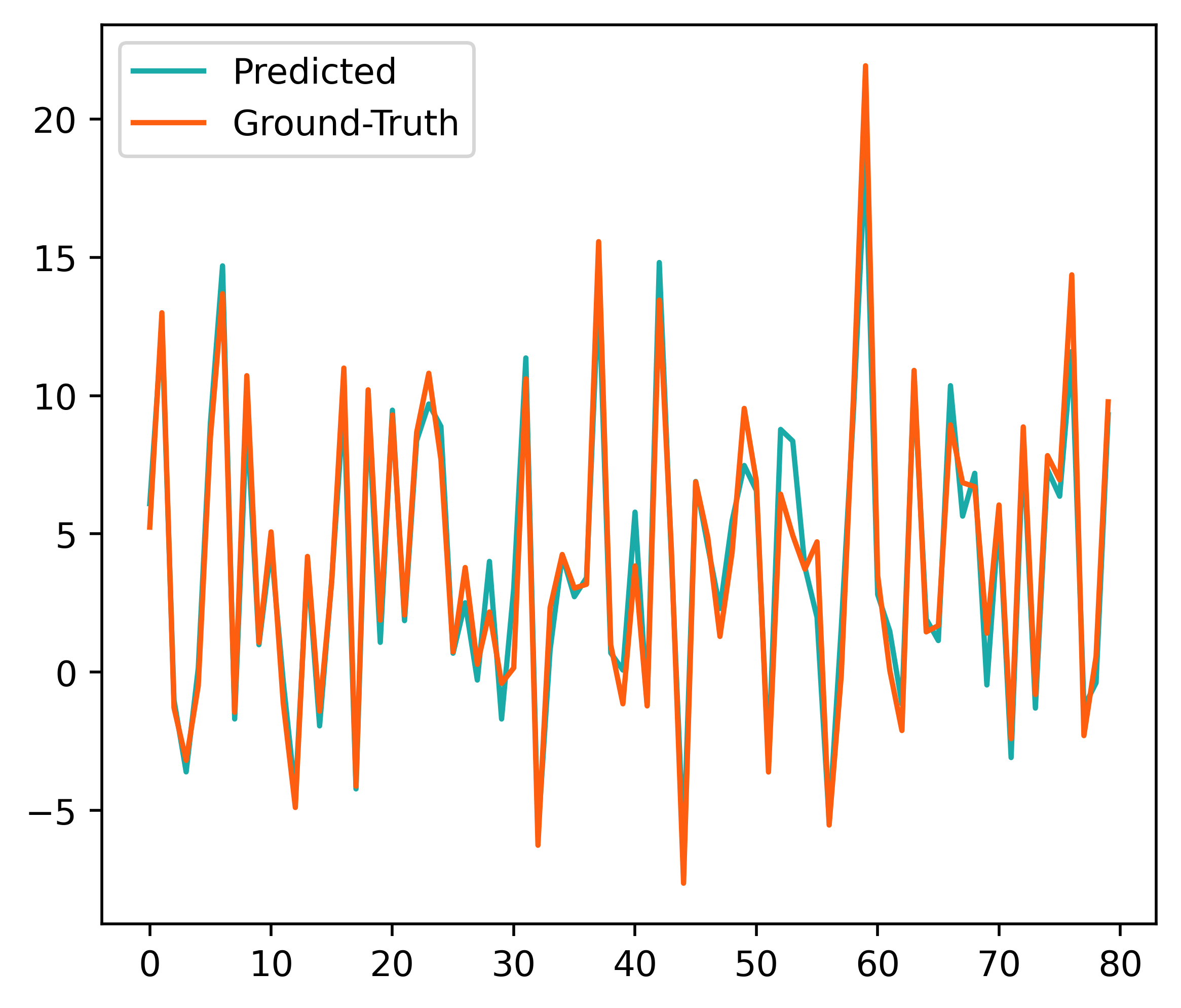}
            \caption{The fitting effects of $X$ and $t$ to $y$.}
            \label{fig:IHDPCase1}
    \end{subfigure}
    \begin{subfigure}[t]{0.45\textwidth}
            \centering
            \includegraphics[width=\textwidth]{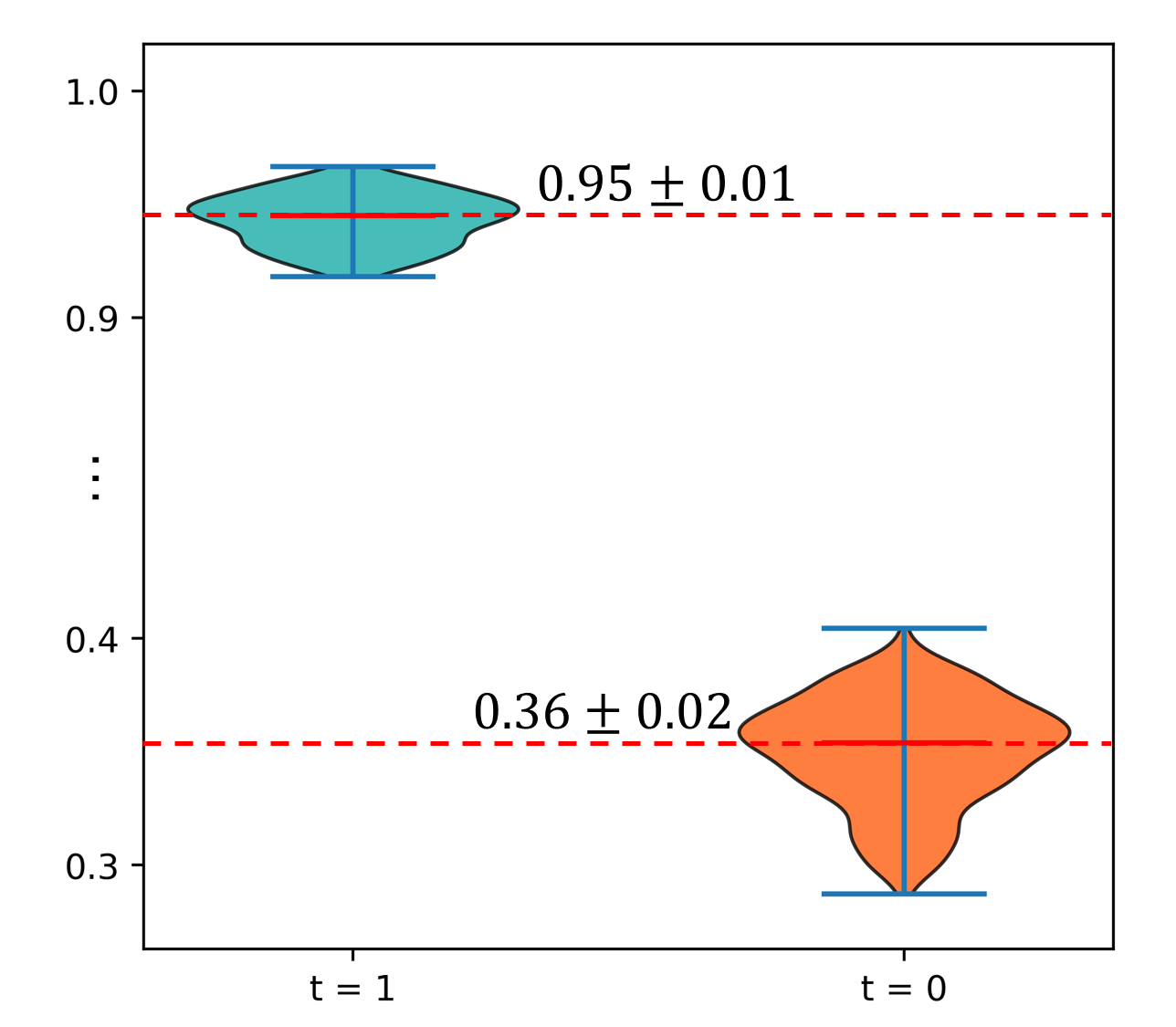}
            \caption{The violin plot of predicted $C^{u}$.}
            \label{fig:IHDPCase2}
    \end{subfigure}
    \caption{The results of applying VLUCI to the IHDP dataset regarding $\hat{y}(X,t)$ and $C^{u}$.}
    \label{FIG:IHDP}
\end{figure}

The insights from the application on dataset IHDP suggest that one should pay attention to the generated distribution of unobserved confounders when applying VLUCI. If the generated distribution is a standard Normal distribution, which may indicate the absence of unobserved confounders, otherwise the VLUCI is valid. Meanwhile, the specific meaning of the generated $C^{u}$ should be understood in the context of the analysis of the model's results at the training process and the actual scenario.
\section{Conclusion}
The presence of unobserved confounders introduces significant uncertainty in counterfactual inference. In this paper, we proposes VLUCI, a novel inference model that derives a reasonable distribution of unobserved confounders while relaxing the untestable assumption of unconfoundedness. The inferred unobserved confounders can be further incorporated into the training of counterfactual inference models to provide a confidence interval estimation for the counterfactual outcomes. The unbiased estimation theoretical proof of causal effects for covariates, along with the VLUCI for inferring unobserved confounders, partially fills a research gap within this subfield. Extensive experiments on synthetic dataset demonstrate that VLUCI has superior performance in generating unobserved confounders and improving the accuracy of counterfactual inferences of state-of-the-art models at both group and individual levels. Moreover, we apply VLUCI to the IHDP dataset which is not highly consistent with the causal assumptions proposed in this paper. Nevertheless, the proposed model still generates an unobserved variable indicating skin color related to the treatment variable, which helps understand the causal mechanism underlying the data and suggests its broad practical applicability. In the future, exploring the actual implications of unobserved confounders generated by the VLUCI in practice is a research topic that deserves attention.

\appendix
\section{Derivation of the optimization objectives for the VLUCI}

The derivation for minimizing the KL divergence shown in Eq.~\eqref{equ8} is described below. Firstly, the observable log-likelihood term $log(p(\acute{t},\acute{y}))$ is decomposed, as shown in Eq.~\eqref{equ59}.
\begin{equation}
\label{equ59}
\tag{A.1}
\begin{aligned}
\mathcal{L}&=log(p(\acute{t},\acute{y}))\\
&= log(p(\acute{t},\acute{y}))\sum_{C^{u}}q(C^{u}|\acute{t},\acute{y}) \\
&= log( \frac{p(\acute{t},\acute{y},C^{u})}{p(C^{u}|\acute{t},\acute{y})})\sum_{C^{u}}q(C^{u}|\acute{t},\acute{y}) \\
&= log( \frac{p(\acute{t},\acute{y},C^{u})}{q(C^{u}|\acute{t},\acute{y})} \frac{q(C^{u}|\acute{t},\acute{y})}{p(C^{u}|\acute{t},\acute{y})} )\sum_{C^{u}}q(C^{u}|\acute{t},\acute{y}) \\
&= log( \frac{p(\acute{t},\acute{y},C^{u})}{q(C^{u}|\acute{t},\acute{y})})\sum_{C^{u}}q(C^{u}|\acute{t},\acute{y}) 
+ log( \frac{q(C^{u}|\acute{t},\acute{y})}{p(C^{u}|\acute{t},\acute{y})} )\sum_{C^{u}}q(C^{u}|\acute{t},\acute{y}) \\
&= \mathcal{L}^{lower} + KL(q(C^{u}|\acute{t},\acute{y})||p(C^{u}|\acute{t},\acute{y})) \\
\end{aligned}
\end{equation}

From Eq.~\eqref{equ59}, $\mathcal{L}$ consists of $\mathcal{L}^{lower} $ and $KL(q(C^{u}|\acute{t},\acute{y})||p(C^{u}|\acute{t},\acute{y}))$ on the right-hand side of the equation. The term $KL(q(C^{u}|\acute{t},\acute{y})||p(C^{u}|\acute{t},\acute{y}))$ is always greater than zero, so there exists a variational lower bound $\mathcal{L}^{lower}$ for $\mathcal{L}$. It should be noted that $\mathcal{L}$ is observable and constant. Therefore the objective of minimizing $KL(q(C^{u}|\acute{t},\acute{y})||p(C^{u}|\acute{t},\acute{y}))$ can be translated into maximizing $\mathcal{L}^{lower}$ which is optimizable. Next, $\mathcal{L}^{lower}$ is further decomposed, as shown in Eq.~\eqref{equ60}.
\begin{equation}
\label{equ60}
\tag{A.2}
\begin{aligned}
\mathcal{L}^{lower}&=log( \frac{p(\acute{t},\acute{y},C^{u})}{q(C^{u}|\acute{t},\acute{y})})\sum_{C^{u}}q(C^{u}|\acute{t},\acute{y})\\
&= log( \frac{p(\acute{t},\acute{y}|C^{u})p(C^{u})}{q(C^{u}|\acute{t},\acute{y})})\sum_{C^{u}}q(C^{u}|\acute{t},\acute{y})\\
&= log( \frac{p(C^{u})}{q(C^{u}|\acute{t},\acute{y})})\sum_{C^{u}}q(C^{u}|\acute{t},\acute{y}) 
+ log( p(\acute{t},\acute{y}|C^{u}))\sum_{C^{u}}q(C^{u}|\acute{t},\acute{y})\\ 
&= -KL(q(C^{u}|\acute{t},\acute{y})||p(C^{u})) 
+\mathbb{E}_{q(C^{u}|\acute{t},\acute{y})}(log( p(\acute{t},\acute{y}|C^{u})))
\end{aligned}
\end{equation}

Eq.~\eqref{equ60} illustrates that the objective function for maximising the variational lower bound $\mathcal{L}^{lower}$ can be decomposed into minimising $KL(q(C^{u}|\acute{t},\acute{y})||p(C^{u}))$ and maximising $\mathbb{E}_{q(C^{u}|\acute{t},\acute{y})}(log( p(\acute{t},\acute{y}|C^{u})))$.

For the objective of minimizing $KL(q(C^{u}|\acute{t},\acute{y})||p(C^{u}))$: based on the introduction of the model structure in Fig.~\ref{Architecture}, $q(C^{u}|\acute{t},\acute{y}) \sim N(\mu,\Sigma)$ can be learned through the second stage of the network. Furthermore, we assume that $p(C^{u})$ subjects to the multivariate Gaussian distribution $N(0,I)$. This assumption is natural. On the one hand, the $C^{u}$ is an exogenous variable which is only affected by random noise. On the other hand, it is based on the statement: "The key is to notice that any distribution in d dimensions can be generated by taking a set of $d$ variables that are normally distributed and mapping them through a sufficiently complicated function."\cite{doersch2016tutorial}. The third objective function based on the $KL(q(C^{u}|\acute{t},\acute{y})||p(C^{u}))$ of VLUCI is shown in Eq.~\eqref{equ61}.
\begin{equation}
\label{equ61}
\tag{A.3}
\mathcal{L}(q(C^{u}|\acute{t},\acute{y}),p(C^{u})) = KL(N(\mu,\Sigma)||N(0,I)))
\end{equation}

For the objective of maximising $\mathbb{E}_{q(C^{u}|\acute{t},\acute{y})}(log( p(\acute{t},\acute{y}|C^{u})))$, it can be noted that is a process of finding the expectation of a log-likelihood function. The expectation is approximated through Monte Carlo(MC) simulation. That is random sampling $C^{u}_{sam}$ from the distribution $q(C^{u}|\acute{t},\acute{y})$ and averaging the log-likelihood values over $log( p(\acute{t},\acute{y}|C^{u}_{sam}))$. Furthermore, since the maximum likelihood estimate is equivalent to the least squares estimation, the objective is transformed into minimizing the reconstruction error demonstrated in Eq.~\eqref{equ62}.
\begin{equation}
\label{equ62}
\tag{A.4}
\max: \mathbb{E}_{q(C^{u}|\acute{t},\acute{y})}(log( p(\acute{t},\acute{y}|C^{u}))) \approx 
\min: MSE(\acute{t},\acute{t}(C^{u}_{sam})) + MSE(\acute{y},\acute{y}(C^{u}_{sam},t))  
\end{equation}
Where, $C^{u}_{sam} \sim p(C^{u}|\acute{t},\acute{y}) = N(\mu,\Sigma)$.


\bibliographystyle{cas-model2-names}

\bibliography{mybib}


\end{document}